\documentclass[twoside,11pt]{article}
\usepackage{times,color,verbatim,lscape,setspace}
\usepackage{graphicx, amsmath, amssymb,threeparttable,subfigure,geometry,enumerate}
\usepackage{times,color,verbatim,lscape,setspace,multirow}
\usepackage{algorithm,algorithmicx}
\usepackage{algpseudocode}
\usepackage{jmlr2e}

\def\boxit#1{\vbox{\hrule\hbox{\vrule\kern6pt
   \vbox{\kern6pt#1\kern6pt}\kern6pt\vrule}\hrule}}

\def\tit.arg{Subsampled Optimization: Statistical Guarantees, Mean Squared Error Approximation, and Sampling Method}

\def\abst.arg{
For optimization on large-scale data, exactly calculating its solution may be computationally difficulty because of the large size of the data. 
In this paper we consider subsampled optimization for fast approximating the exact solution. In this approach, one gets a surrogate dataset by sampling from the full data, and then obtains an approximate solution by solving the subsampled optimization based on the surrogate. 
One main theoretical contributions are to provide the asymptotic properties of the approximate solution with respect to the exact solution as statistical guarantees, and to rigorously derive an accurate approximation of the \emph{mean squared error (MSE)} and an approximately unbiased MSE estimator.
These results help us better diagnose the subsampled optimization in the context that a confidence region on the exact solution is provided using the approximate solution. 
The other consequence of our results is to  propose an optimal sampling method, Hessian-based sampling, whose probabilities are proportional to the norms of Newton directions. Numerical experiments with least-squares and logistic regression show promising performance, in line with our results. 
}

\def\key.arg{
Large-scale Optimization, Random Sampling, Statistical Guarantee, Inference, Sampling-dependent, Mean Squared Error, Confidence Region}

\def\authorhere{Rong Zhu$^a$ and Jiming Jiang$^b$ 
\vskip 0.2cm
{\it $^a$Academy of Mathematics and Systems Science, Chinese Academy of Sciences, Beijing 100190, China}\\
{\rm rongzhu@amss.ac.cn}\\
\medskip
{\it $^b$ Department of Statistics, University of California, Davis, CA 95616, U.S.A.} \\
{\rm jimjiang@ucdavis.edu}

}

\begin{document}

\begin{center}
{\Large \tit.arg}\footnote{
Rong Zhu's work was partially supported by Natural Science Foundation of China grant 11301514.
Jiang's research  was partially supported by NSF grant DMS-1510219 and
the NIH grant R01-GM085205A1.
}


\authorhere
\end{center}

\centerline{\small SUMMARY}
\abst.arg


\section{Introduction}

Optimization problems in machine learning and statistics are based on some sort of minimization problems, that is, a parameter of interest is solved by minimizing an objective function defined by a loss function over a data set. These problems cover many methods, for example, maximum likelihood and least squares estimation (\cite{Lehmann:03}) as well as numerical optimization (\cite{Boyd:04, Hazan:06, Nocedal:06}). 
In modern large-scale data analysis, 
we often meet the exploding sample size. 
For instance, in applying Newton's method to the problem involving an $N\times d$ data matrix where $N$ is data size and $d$ is the dimension, each iteration has complexity scaling as $O(Nd^2)$ assuming $N>d$. Exactly solving this problem often is computationally difficulty when the sample size $N$ is very large. 
For dealing with large-scale data, 
random sampling provides a powerful way. 
Instead of exactly solving it, one draws a small subsampled data from the original data by some sampling method, and then carries out computations of interest from the subsampled data. 
The benefits of this approach is to greatly save the computational cost and to eliminate the need for storing the full data. 
Recently, this approach has been applied to many problems in large-scale data analysis, e.g, least-squares approximation (\cite{Drineas:06, Drineas:10, Zhu:16}), low-rank matrix approximation (\cite{FKV04,DMM08,Bach:13}), large-scale boosting (\cite{Freund:96, Dubout:14}) and bootstrap (\cite{Kleiner:14}). 


In this paper we propose \emph{subsampled optimization} in order to approximately solve the optimization problem on large-scale data. 
By the subsampled optimization, the \emph{approximate solution} is obtained based on a subsampled data randomly drawn from the full data in such a way that the computational cost of solving it is associated with the subsampled data size. 
For the example of Newton's method on the $N\times d$ data matrix, when we draw the subsampled data of size $n$ ($d < n < N$), the complexity per iteration is $O(nd^2)$. 
If the complexity time of the sampling process can be scalable to $O(Nd)$ (\cite{Drineas:12,Zhu:16}), the subsampled optimization greatly reduces the computation cost. 
Besides the computational benefit, this approach has been proved to have nice statistical performance with guarantees. 
In previous studies, these guarantees are based on convergence rates, where lower/upper error bounds are usually driven (such as \cite{Erd:15}, \cite{PW:16}, and \cite{PW:17}). 
However,  error bounds may be far less accurate diagnosis for the subsampled optimization, as they just provide margin measure.  Another drawback is that error bounds may become unusable in practice, since they often reply on unknown constants.

In contradistinction to convergence rates, in this paper we provide several new theoretical contributions from the statistical perspectives. 
First, asymptotic properties are provided as the statistical guarantees of the subsampled optimization. 
More concretely, we provide the consistency, unbiasedness and asymptotic normality of the approximate solution based on the subsampled data with respect to the exact solution based on the full sample data. 
Second, in order to sufficiently diagnose the subsampled optimization, we \emph{rigorously} develop the mean squared error (MSE) approximation, and then supply the MSE estimation based on the subsampled data.
To do so, we divide the MSE into two terms: the main part that is expressed as a sandwich-like term, and the negligible remainder part that decays as $O(n^{-3/2})$.
Furthermore, we show that the sandwich-like term is estimable from the subsampled data. It follows an approximately unbiased  MSE estimator.
To the best of our knowledge, this is the first work on rigorous derivation of the MSE approximation and the MSE estimation in the sampling-based approximation problems. 
Our proofs rely on sampling techniques rather than the classical likelihood theory (\cite{Lehmann:03}) that assumes the data generating process of the units.
Unlike the classical likelihood theory, our results do not reply on any assumption of population distribution generating data, but quantify the randomness from the sampling process. 
Thus, some of our theoretical proofs are non-standard and may be of independent interest.

Our theoretical results bring some valuable consequences. One implication is  to construct confidence region on the exact solution of the original optimization by using the approximate solution of our subsampled optimization.  
As far as we know, this is the first work to introduce confidence region into sampling-based approximation problems. 
This statistical inference provides an accurate measure of confidence on the large-scale optimization from the subsampled optimization. 
The other implication is to provide an efficient sampling method, the computational complexity of which is scalable. 
The sampling probabilities are proportional to the norms of Newton directions used in Newton's method in optimization.   
We call it as Hessian-based sampling, because it makes use of Hessian matrix information in the sampling process. 
These implications demonstrate the potential power of our theoretical analysis over the existing convergence analysis. 
Empirically, numerical experiments with least-squares and logistic regression in Section \ref{sec:num} show promising performance in line with our results. 

{\bf Related work.} The subsampled optimization is related to but different from resampling methods, such as bootstrapping method (\cite{Efron:79}) and subsampling method (\cite{Poli:Roma:Wolf:subs:1999}). The goal of these resampling methods is to make simulation-based inference, which traditionally uses repeated computation of estimates of resamples (or subsample), although the subsampling method (\cite{Poli:Roma:Wolf:subs:1999}) allows these resamples to be significantly smaller than the original data set. Unlike these resampling methods, our aim is to approximately solve the large-scale optimization by our subsampled optimization. 
There are some studies on investigating optimization algorithms based on subsampled data (\cite{Erd:15, PW:17}). Unlike them, our results do not reply on specific algorithms, 
but provide a general theoretical analysis of the subsampled optimization.  

{\bf Notation:} For a vector $v\in \mathbb{R}^d$, we use $\|v\|$ to denote its Euclidean norm, that is, $\|v\|=(\sum_{j=1}^dv_i^2)^{1/2}$. For a matrix $A\in \mathbb{R}^{d_1\times d_2}$, we use $\|A\|_2$ to denote its spectral norm, that is, $\|A\|_2=\sup_{x\in \mathbb{R}^{d_2}; \|x\|\leq 1}\|Ax\|$. The notation $A=O(n^{\kappa})$, where $\kappa$ is a constant, means that each element of $A$ is the order $O(n^{\kappa})$. 
For function $f(\theta; \cdot)$, $\nabla f(\theta; \cdot)=\frac{\partial }{\partial\theta} f(\theta; \cdot)$ and $\nabla^2 f(\theta; \cdot)=\frac{\partial^2 }{\partial\theta\partial\theta^{\top}} f(\theta; \cdot)$. 
Throughout this paper, the notation ``$\text{E}$" means expectation under the sampling process, which is where the randomness comes from.
All theorems are proved in the appendices. To save notations, we liberally share the bounding constants using the notation $C$ in proofs.

\section{Problems set-up and assumptions}
\label{sec:m}
We begin by setting up optimization from a large-large dataset; this is followed by subsampled optimization based a subsample drawn from the dataset by some sampling method. 

\subsection{Background}
Let $\{f(\cdot; x), x\in \mathcal{X}\}$ be a collection of real-valued and convex loss functions, each defined on a finite dimensional set containing the convex set $\Theta\subset\mathbb{R}^d$, and $U=\{1, 2, \cdots, N\}$ be a large dataset, where for each $i\in U$, a data point $x_i\in \mathcal{X}$ is observed. The empirical risk from the dataset $F_U: \Theta\rightarrow \mathbb{R}$ is given by 
\begin{equation}\label{risk}
F_U(\theta)=\frac{1}{N}\sum_{i\in U}f(\theta; x_i).
\end{equation}
Our goal is to obtain the solution $\hat{\theta}_N$ by minimizing the risk, namely the quantity
\begin{equation}\label{theta-N}
\hat{\theta}_N=\arg\min_{\theta\in \Theta}F_U(\theta).
\end{equation}
Throughout the paper, we impose some regularity conditions on the parameter space, the risk function and the loss function. \\
\noindent\textbf{Assumption A (Parameters):} The parameter space $\Theta\subset\mathbb{R}^d$ is a compact convex set with $\hat{\theta}_N\in \text{ the interior of }\Theta$ and $l_2$-radius $R$.\\  
\noindent \textbf{Assumption B (Convexity):} The risk function $F_U(\theta)$ is twice differentiable and $\lambda$-strongly convex over $\Theta$, that is, for $\theta\in \Theta$,
\begin{equation*}
\nabla^2 F_U(\theta)\geq \lambda I,
\end{equation*} 
where $\geq$ denotes the semidefinite ordering.\\
\noindent \textbf{Assumption C (Smoothness):} 
The gradient vector $\nabla f(\theta; x_i)$ and the Hessian matrix $\nabla^2 f(\theta; x_i)$ are $L(x)$-Lipschitz continous, that is, 
for all $\theta_1, \theta_2 \in \Theta$ and for any $x$, 
\begin{align*}
&\left\|\nabla f(\theta_1;x)-\nabla f(\theta_2;x)\right\|\leq L(x)\|\theta_1-\theta_2\|,\\
&\left\|\nabla^2f(\theta_1;x)-\nabla^2f(\theta_2;x)\right\|_2\leq L(x)\|\theta_1-\theta_2\|.
\end{align*}
Assumption A is a standard condition to simplify the arguments used in the proofs. Assumptions B \& C require that the loss function is convex and smooth in a certain way. Many loss functions satisfy these conditions including, for example, the $l_2$ loss function for linear regression and logistic function in logistic regression. Similar conditions are used in optimization problems (\cite{BV:04, Zhang:13}). 
However there are some cases, such as $l_1$ loss function, which do not satisfy Assumptions B \& C. It is an unaddressed problem to relax them to consider broader class of optimization problems.

Under the assumptions above, $\hat{\theta}_N$ satisfies the system of equations
\begin{equation}\label{system-equ}
\nabla F_U(\theta):=\frac{1}{N}\sum_{i\in U}\nabla f(\theta; x_i)=0.
\end{equation}

\subsection{Subsampled optimization}
For fast solving large-scale optimization in Eqn.(\ref{theta-N}), we propose the \emph{subsampled optimization}. 
Consider a collection $S$ of size $n$ ($n<N$) that is drawn with replacement from $U$ according to a sampling probability $\{\pi_i\}_{i=1}^N$ such that $\sum_{i=1}^N\pi_i=1$.  
\emph{The subsampled optimization} is meant to obtain the solution $\hat{\theta}_n$ satisfying the following system of equations 
\begin{equation}\label{theta-n}
\nabla F_S(\theta)=0,  
\text{ where } F_S(\theta)=\frac{1}{Nn}\sum_{i\in S}\frac{1}{\pi_i}f(\theta; x_i).
\end{equation}
In Eqn.(\ref{theta-n}), we construct $F_S(\theta)$ by inverse probability weighting in such a way that $\text{E}[F_S(\theta)]=F_U(\theta)$. 
The idea of the subsampled optimization is to minimize the unbiased risk function $F_S(\theta)$ in place of $F_U(\theta)$. 
However, it is not clear about the performance of $\hat{\theta}_n$ with respect to $\hat{\theta}_N$. 
In this paper we investigate theoretical performance of $\hat{\theta}_n$ with respect to $\hat{\theta}_N$, and supply the consistency, asymptotic unbiasedness and asymptotic normality in Sections \ref{sec:asymp}. 

\begin{remark}
Empirically, we do numerical comparison  with the subsampled optimization with equal weighting, which is defined as obtaining the solution $\tilde{\theta}_n$ satisfying 
\begin{equation*}
\nabla \tilde{F}_S(\theta)=0,
\text{ where } \tilde{F}_S(\theta)=\frac{1}{n}\sum_{i\in S}f(\theta; x_i).
\end{equation*}
Unlike $F_S(\theta)$, the risk function $\tilde{F}_S(\theta)$ is directly constructed from the subsampled data without the probability weighting.  
The empirical results show that $\hat{\theta}_n$ is much more robust than $\tilde{\theta}_n$. See details in Appendix \ref{Addit-simulation}. 
\end{remark}

To guarantee good properties of the subsampled optimization, we require some regularity conditions on sampling-based moments.\\
\noindent\textbf{Assumption D (Sampling-based Moments):}
There exists the following sampling-based moments condition: for $k=2,4$, 
\begin{align*}
&N^{-k}\sum_{i=1}^N\frac{1}{\pi_i^{k-1}}\left\|\nabla f(\theta_N; x_i)\right\|^k=O(1),  \ \ 
N^{-k}\sum_{i=1}^N\frac{1}{\pi_i^{k-1}}\left\|\nabla^2 f(\theta_N; x_i)\right\|_2^k=O(1),\\
&N^{-k}\sum_{i=1}^N\frac{1}{\pi_i^{k-1}}(L(x_i))^k=O(1). 
\end{align*}
Assume $\pi_i$s are $\alpha$-tolerated, meaning that $\min\{N\pi_i\}_{i=1}^N\geq \alpha$ for some constant $\alpha$.  This condition is common in statistics (\cite{Fuller:09, Breidt:00}).
Under this condition, Assumption D is implied by that
\begin{align*}
&\frac{1}{N}\sum_{i=1}^N\left\|\nabla^2 f(\theta_N; x_i)\right\|_2^4=O(1),  \ \
\frac{1}{N}\sum_{i=1}^N\left\|\nabla f(\theta_N; x_i)\right\|^4=O(1), \ \ \frac{1}{N}\sum_{i=1}^N(L(x_i))^4=O(1).
\end{align*}

\section{Our results}
\label{sec:results}
Having described the subsampled optimization procedure, we now state our main theorems. 
The main theoretical results of this paper is to accurately diagnose how well the subsampled optimization defined in Eqn. (\ref{theta-n}) approximates the large-scale optimization in Eqn. (\ref{theta-N}).  We do so by rigorously analyzing the difference between the approximate solution $\hat{\theta}_n$ of the subsampled optimization  and the exact solution $\hat{\theta}_N$ of the large-scale optimization. By this rigorous analysis,  we provide the asymptotic properties, the MSE approximation, and the MSE estimation.
Our analysis formulation is essentially different from classical likelihood theory, which make assumptions on the data generation process (\cite{Lehmann:03}).
Unlike the classical likelihood theory, our statistical analysis requires no assumptions on the data generating process of the units, 
but quantifies the approximation error from sampling randomness.  
Therefore, our results, which rely on sampling techniques rather than likelihood theory, reveal the effect of the sampling process on the subsampled optimization. 
Some of our proof techniques are non-standard so could be of independent interest. 

\subsection{Asymptotic Properties}
\label{sec:asymp}

\begin{theorem}\label{them:asymptotic}
Suppose that Assumptions A--D are satisfied. Then, we have 
\begin{align}
&\|\text{E}\hat{\theta}_n-\hat{\theta}_N\|=O(n^{-1}). \label{unbiasedness}\\
&\text{E}\|\hat{\theta}_n-\hat{\theta}_N\|^2=O(n^{-1}), \label{consistency}
\end{align}
Furthermore, define $\Sigma=\nabla^2F_U(\hat{\theta}_N)$, and 
\begin{align}\label{AMSE-expr}
\text{AMSE}(\hat{\theta}_n)
=&\Sigma^{-1}\left[\frac{1}{n}\frac{1}{N^2}\sum_{i \in U}\frac{1}{\pi_i}\nabla f(\hat{\theta}_N;x_i)\nabla^{\top}f(\hat{\theta}_N;x_i)\right]\Sigma^{-1}.
\end{align} 
We have that, 
 as $n\rightarrow\infty$, given the large-scale dataset $U$,
  \begin{equation}\label{normal}
    \text{AMSE}(\hat{\theta}_n)^{-1/2}(\hat{\theta}_n-\hat{\theta}_N) 
    \rightarrow N(0,I) \text{ in distribution. }
  \end{equation}
\end{theorem}

The results of Theorem \ref{them:asymptotic} are the asymptotic properties of the subsampled optimization with respect to the optimization on large-scale data. 
Eqn.(\ref{unbiasedness}) tells us that $\hat{\theta}_n$ is approximately unbiased with respect to $\hat{\theta}_N$. 
Combing Eqns.(\ref{unbiasedness}) \& (\ref{consistency}), the variance-bias tradeoff is provided: the variance term is the dominant component in the MSE, since the squared bias term that is the order $O(n^{-2})$ is negligible relative to the MSE.
Another significant implication of Eqn.(\ref{consistency}) is that, by Markov's inequality, we have, as $n\rightarrow\infty$, 
\begin{equation*}
\hat{\theta}_n-\hat{\theta}_N\rightarrow 0 \ \ \text{ in probability}.
\end{equation*} 
It means that $\hat{\theta}_n$ is consistent to $\hat{\theta}_N$.
Eqn.(\ref{normal}) shows that $\hat{\theta}_n-\hat{\theta}_N$ goes to normal in distribution as the subsampled data size $n$ goes large. 

\subsection{Mean squared error approximation}
Now we investigate the mean squared error (\emph{MSE}) of the solution from the subsampled optimization, which defined as 
\begin{equation}\label{MSE-def}
\text{MSE}(\hat{\theta}_n):=\text{E}(\hat{\theta}_n-\hat{\theta}_N)(\hat{\theta}_n-\hat{\theta}_N)^{\top}.
\end{equation}
Note that the usual notation of MSE is defined as $\text{E}\|\hat{\theta}_n-\hat{\theta}_N\|^2$, while we extend the notation into the matrix form here. 
Now we derive a valid approximation of $\text{MSE}(\hat{\theta}_n)$. In the following theorem, we divide the MSE into two terms: the leading term with an explicit expression and the remaining term that decays as lower order. 
\begin{theorem}\label{them-MSE}
 Suppose that Assumptions A--D are satisfied. Then, we have
\begin{align}\label{MSE}
\text{MSE}(\hat{\theta}_n)&=\text{AMSE}(\hat{\theta}_n)+O(n^{-3/2}).
\end{align}
\end{theorem}
Theorem \ref{them-MSE} shows that $\text{MSE}(\hat{\theta}_n)$ can be divided into the leading term $\text{AMSE}(\hat{\theta}_n)$ and the remainder term that decays as $O(n^{-3/2})$.  We have shown that $\text{MSE}(\hat{\theta}_n)=O(n^{-1})$  in Eqn.(\ref{consistency}) of Theorem \ref{them:asymptotic}, so the remainder term is negligible. 


The key contribution of our MSE approximation is to obtain an expressible measure for accurately knowing the performance of the subsampled optimization with respect to the large-scale optimization.
This measure is different from existing analysis on convergence rates, which just provides a margin measure so probably is far less enough to provide a sufficient diagnosis for the subsampled optimization.
Our MSE approximation sufficiently diagnoses the effect of sampling. 
We shall show some valuable implications in Section \ref{sec:conseq}.

\begin{remark}
From Theorem \ref{them-MSE}, the mean squared prediction error (MSPE) can be approximated. Given a smooth function $g(\theta)$,  the MSPE of $g(\hat{\beta}_n)$ with respect to $g(\hat{\beta}_N)$ is easily obtained from MSE as follows: $$E[g(\hat{\theta}_n)-g(\hat{\beta}_N)]^2\approx\nabla^{\top} g(\hat{\theta}_N)\text{AMSE}(\hat{\theta}_n)\nabla g(\hat{\theta}_N).$$
\end{remark}

\subsection{Mean squared error estimation}
Although $\text{AMSE}(\hat{\theta}_n)$ can approximate the MSE, calculating it from the full large-scale data meets computational bottleneck. To address this problem, 
we now investigate the MSE estimation from the subsampled data. The goal is to obtain such an estimator $\text{mse}(\hat{\theta}_n)$ based on the subsampled data that is approximately unbiased in the sense that $\text{E}[\text{mse}(\hat{\theta}_n)]=\text{MSE}(\hat{\theta}_n)+\text{remainder}$. This means that the bias of $\text{mse}(\hat{\theta}_n)$ in estimating $\text{MSE}(\hat{\theta}_n)$ is negligible.
Define $\hat{\Sigma}=\nabla^2F_S(\hat{\theta}_n)$, and construct the MSE estimator as
\begin{align}\label{mse-formula}
&\text{mse}(\hat{\theta}_n)
=\hat{\Sigma}^{-1}\left[\frac{1}{n^2}\frac{1}{N^2}\sum_{i\in S}\frac{1}{\pi_i^2}\nabla f(\hat{\theta}_n;x_i)\nabla^{\top}f(\hat{\theta}_n;x_i)\right]\hat{\Sigma}^{-1}.
\end{align}
Define $$\mathcal{E}_F:=\{\|\nabla^2 F_S(\theta)-\nabla^2 F_U(\theta)\|_2\leq c\lambda\} \text{ for }\theta\in \Theta, 0<c<1,$$ 
and denote $\mathcal{E}_F^c$ as the complement set of $\mathcal{E}_F$. In Lemma \ref{e2-c} of Appendix \ref{sec:lemma}, we prove that $$Pr(\mathcal{E}_F^c)=O(n^{-1}).$$
In Assumption B we require 
$\nabla^2 F_U(\theta)\geq \lambda I$ for $\theta\in \Theta$, so 
given the event $\mathcal{E}_F$ holds, $\nabla^2 F_U(\theta)\geq (1-c)\lambda I$ for $\theta\in \Theta$. 
This implies that $\hat{\Sigma}$ is invertible if $\mathcal{E}_F$ holds. 
\begin{theorem}\label{them-mse}
Suppose that the Assumptions A--D are satisfied. Then, we have that, given the event $\mathcal{E}_F$ holds,
\begin{equation}\label{mse-property}
\text{E}[\text{mse}(\hat{\theta}_n)]=\text{MSE}(\hat{\theta}_n)+O(n^{-3/2}).
\end{equation}
\end{theorem}

Theorem \ref{them-mse} shows that $\text{mse}(\hat{\theta}_n)$ estimates $\text{MSE}(\hat{\theta}_n)$ well, so $\text{mse}(\hat{\theta}_n)$ is used to diagnose the performance of the subsampled optimization from the subsampled data. 
It also provides a practical way to choose the subsampled data size:  the size $n$ is decided via calculating $\text{mse}(\hat{\theta}_n)$ from the subsampled data in order to meet a precision requirement,

\section{Some consequences}
\label{sec:conseq}
We now turn to deriving some useful consequences of our main theorems for applying the subsampled optimization. One application is to construct confidence region on $\hat{\theta}_N$. 
The other is to provide an efficient way for sampling a subsampled dataset. 
\subsection{Confidence region on $\hat{\theta}_N$}
\label{sec:CI}
We introduce the confidence region into the subsampled optimization to construct a confidence region on $\hat{\theta}_N$. 
Due to Eqn. (\ref{normal}) in Theorem \ref{them:asymptotic},
$$(\hat{\theta}_n-\hat{\theta}_N)^{\top}[\text{AMSE}(\hat{\theta}_n)]^{-1}(\hat{\theta}_n-\hat{\theta}_N)\rightarrow \chi_d^2  \text{ in distribution,}$$
where the degree of freedom $d$ is the size of the coefficients. 
Since $\text{mse}(\hat{\theta}_n)$ is an approximately unbiased estimator of $\text{MSE}(\hat{\theta}_n)$ from Theorem \ref{them-mse}, we construct the ellipsoidal confidence region $\Im$ for $\hat{\theta}_N$ as follows:
\begin{equation}\label{CI}
\Im=\left\{(\hat{\theta}_n-\hat{\theta}_N)^{\top}[\text{mse}(\hat{\theta}_n)]^{-1}(\hat{\theta}_n-\hat{\theta}_N)\leq \chi_d^2(q)\right\}.
\end{equation}
where $q$ is the confidence level, e.g. 0.95, that the region contains $\hat{\theta}_N$.
From Eqn.(\ref{CI}), given a confidence level $q$, the confidence region for $\hat{\theta}_N$ is constructed around $\hat{\theta}_n$. 
This statistical inference provides an accurate measure of confidence on the large-scale optimization based on the subsampled optimization. 
We do some numerical experiments in Section \ref{sec:num}.

\subsection{Optimal sampling: Hessian-based sampling}
\label{sec:application}

One contribution of Theorem \ref{them-MSE} is to guide us to find some efficient sampling process such that $\text{AMSE}(\hat{\theta}_n)$ attains its minimum in some sense.
From Eqn.(\ref{AMSE-expr}) and Cauchy-Schwartz inequality, if $\pi_i$ takes the value such that
\begin{equation}\label{opt-sampling}\pi_i^{opt}\propto \|\Sigma^{-1}\nabla f(\hat{\theta}_N;x_i)\|,\end{equation}
then $\text{AMSE}(\hat{\theta}_n)$ attains its minimum, meaning that $\pi_i^{opt}$ is an optimal choice in this sense. 

Compared with the gradient-based sampling proposed by \cite{Zhu:16}, where $\pi_i\propto \|\nabla f(\hat{\theta}_0;x_i)\|$ with a pilot $\hat{\theta}_0$ of $\hat{\theta}_N$, 
our optimal sampling probabilities $\{\pi_i^{opt}\}_{i=1}^N$ are proportional to the norms of Newton directions used in Newton's method.
The optimality of $\{\pi_i^{opt}\}_{i=1}^N$ makes sense because Newton directions make use of the second-order Hessian information and the gradient information.
Another advantage of our sampling method is that the sampling probabilities have the scale invariance property: $\{\pi_i^{opt}\}_{i=1}^N$ do not change if scales $x$ are multiplied by a factor.

However $\{\pi_i^{opt}\}_{i=1}^N$ are impractical as they include the unknown $\hat{\theta}_N$ and $\Sigma$.  
Following the idea of \cite{Zhu:16}, we obtain the pilots $\hat{\theta}_0$ and $\hat{\Sigma}_0$ which are calculated based on a small fraction $S_0$ of size $n_0$ (usually $n_0\leq n$) drawn from $U$ by uniform sampling. Given the dataset $S_0$, the system of equations is formed via $F_{S_0}(\theta)=\frac{1}{n_0}\sum_{i\in S_0}f(\theta; x_i)$.
$\hat{\theta}_0$ is obtained by solving $\nabla F_{S_0}(\theta)=0$ and $\hat{\Sigma}_0=\nabla^2F_{S_0}(\hat{\theta}_0)$.
Then we replace $\hat{\theta}_N$ and $\Sigma$ with their pilots $\hat{\theta}_0$ and $\hat{\Sigma}_0$, respectively.  The optimal sampling used in practice are given below
\begin{equation}\label{opt-sampling-use}\pi_i^{o}\propto \|\hat{\Sigma}_0^{-1}\nabla f(\hat{\theta}_0;x_i)\|.\end{equation}
We call it as the Hessian-based sampling, which is summarized in Algorithm \ref{alg-hessian}.
The computational complexity of this algorithm are from three parts. The first is calculating $\hat{\theta}_0$ and $\hat{\Sigma}_0$. It costs relatively small because we solve them based on a small fraction $S_0$ of size $n_0$ drawn from $U$ by uniform sampling. The second is calculate sampling probabilities $\pi_i^{o}$. It is scalable to $O(Nd)$ considering the example of Newton's method on $N\times d$ data matrix. The third is from solving the subsampled optimization based on the subsampled data in such a way that solving it is associated with the subsampled data size. Therefore, the cost of the subsampled optimization is much smaller than the original optimization. 
\begin{algorithm}
\textbf{Input:} dataset $U$; pilot data size $n_0$; subsampled data size $n$\\
\textbf{Output:} the solution of subsampled optimization $\hat{\theta}_n$\\
\textbf{\emph{Step 1: obtaining sampling probabilities}}: \\
 (1a) Solve initial solution $\hat{\theta}_0$ and $\hat{\Sigma}_0$ obtained based on a small fraction $S_0$ of data of size $n_0$ drawn from $U$ by uniform sampling;\\
 (1b) Calculate the optimal weights $\{\pi_i^{o}\}_{i=1}^N$ for each data point according to Eqn. (\ref{opt-sampling-use});\\
\textbf{\emph{Step 2: sampling the subsampled data}}: \\
  (2) Obtain the subsampled data $S$ of size $n$ via sampling with replacement according to the optimal weights from $U$; \\
\textbf{\emph{Step 3: subsampled optimization}}: \\
  (3) Solve the approximate solution $\hat{\theta}_n$ from the subsampled data according to Eqn.(\ref{theta-n}).
\caption{Hessian-based sampling}\label{alg-hessian}
\end{algorithm}

\section{Empirical observations}
\label{sec:num}
In this section, we carry out numerical studies by simulated data sets to verify our theoretical results  and their implications.. We approximate the exact solvers of coefficients of linear/logistic regressions $\hat{\theta}_N$ by their subsampled solution $\hat{\theta}_n$.  We measure the performance in the sense of the MSE $\text{E}\|\hat{\theta}_n-\hat{\theta}_N\|^2$. 

We first describe the set-up of generating datasets. For the least-squares problem, we generate the dataset of size $N=100K$ under the model: $y_i=x_i^{\top}\theta+\epsilon_i$, where $\theta=(1,1,1,0.1,0.1)^{\top}$, $\epsilon_i\sim N(0,10)$; three quarters of $x_i\sim N(0,\Sigma)$ and one quarter of $x_i\sim N(0,4\Sigma)$, where the $(i,j)$ element of $\Sigma$ is $\Sigma_{ij}=0.5^{|i-j|}$.
For the logistic regression, we generate the dataset of size $N=100K$ from the model: $Pr(y_i=1)=[1+\exp(-x_i^{\top}\theta)]^{-1}$, where $x_i\sim N(0,\Sigma_l)$, and $\Sigma_l=\text{diag}(1,1,1,5,5)$. Although the setting of datasets is not large-scale, it is enough to clearly assess our results by showing valuable empirical observations.
Note that we consider the intercept term in fitting these models. 
We compare four sampling methods: uniform sampling (UNIF), leverage-based sampling (LEV) (\cite{Drineas:06}), gradient-based sampling (GRAD) (\cite{Zhu:16}), and Hessian-based sampling (Hessian).
One thousand replicated samples are drawn via UNIF, LEV, GRAD and Hessian, respectively, with different sampling fractions $0.005, 0.01, 0.02, 0.04, 0.08$. 

\begin{figure*}[h]
\centering
\subfigure[MSE-linear]{
\label{linear-rate}
  \includegraphics[width=0.4\textwidth]{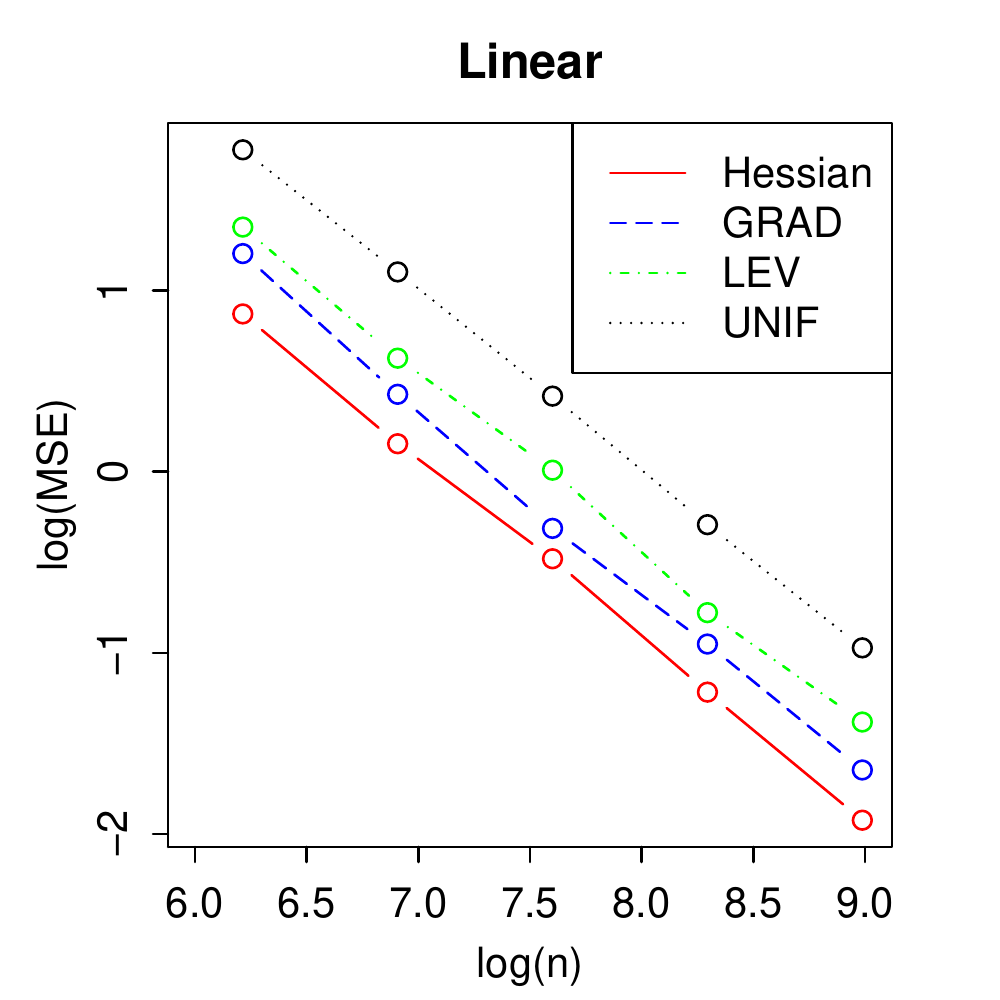}}
  \subfigure[MSE-logistic]{
\label{logistic-rate}
  \includegraphics[width=0.4\textwidth]{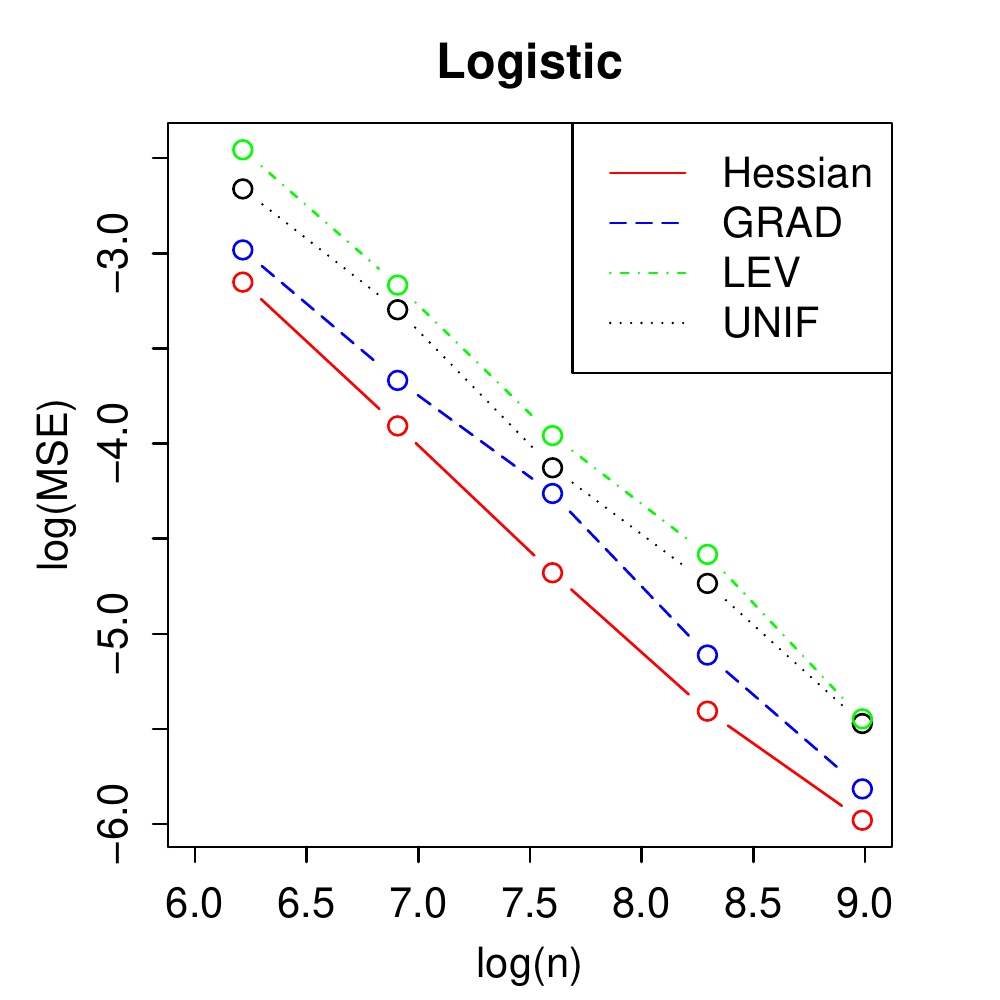}}
\caption{Decreasing Rate of MSE. From left to right: linear regression and logistic regression. 
}
\label{MSE-rate}
\end{figure*}

\begin{table}[!h]
\caption{Ratios between AMSE and MSE for linear and logistic regressions. The 1st row denotes the sampling fraction.}
\begin{center}
\begin{tabular}{ c | c c c c c | c c c c c}
\hline
 sampling & 0.005 & 0.01 & 0.02 & 0.04 & 0.08 & 0.005 & 0.01 & 0.02 & 0.04 & 0.08\\
\hline
& \multicolumn{5}{c|}{Linear regression } &  \multicolumn{5}{c}{Logistic regression} \\\hline
 UNIF & 0.937 & 0.953 & 0.956 & 1.028 & 0.985 &  0.887 & 0.938 & 0.958 & 0.986 & 1.000 \\
 LEV & 0.972 & 0.953 & 1.031 & 1.015 & 1.003 & - & - & - & - & - \\
 GRAD & 0.910 & 0.972 & 0.960 & 1.019 & 0.987 & 0.939 & 1.018 & 0.986  & 1.013 & 0.994 \\
 Hessian & 0.932 & 0.953 & 1.020 & 1.011 & 0.995 & 0.957 & 0.977 & 0.972 & 1.009 & 0.989 \\
\hline
\end{tabular}
\end{center}
\label{Table-AMSEratio}
\end{table}

\begin{table}[!h]
\caption{Ratios between means of MSE estimators and MSE for linear and logistic regressions. The 1st row denotes  the sampling fraction.}
\begin{center}
\begin{tabular}{ c | c c c c c | c c c c c}
\hline
 sampling & 0.005 & 0.01 & 0.02 & 0.04 & 0.08 & 0.005 & 0.01 & 0.02 & 0.04 & 0.08 \\
\hline
& \multicolumn{5}{c|}{Linear regression } & \multicolumn{5}{c}{Logistic regression} \\\hline
 UNIF & 1.014 & 0.919 & 1.004 & 1.049 & 1.024 & 0.956 & 0.969 & 0.988 & 1.014 & 0.971 \\
 LEV & 0.908 & 0.978 & 1.063 & 1.032 & 1.010 & - & - & - & - & - \\
 GRAD & 0.955 & 1.033 & 0.967 & 0.982 & 0.984 & 0.998 & 0.995 & 1.017 & 0.992 & 0.997 \\
 Hessian & 0.988 & 0.995 & 0.935 & 0.982 & 0.997  & 0.898 & 0.976 & 0.972 & 1.013 & 0.999 \\
\hline
\end{tabular}
\end{center}
\label{Table-MSEratio}
\end{table}

\begin{table*}[!h]
\caption{Confidence region performance. The 1st row denotes  the sampling fraction.}
\begin{center}
\begin{tabular}{c | c c c c c  | c c c c c}
\hline
sampling & 0.005 & 0.01 & 0.02 & 0.04 & 0.08 & 0.005 & 0.01 & 0.02 & 0.04 & 0.08\\
\hline
& \multicolumn{5}{c}{Linear; CI: 90\%} & \multicolumn{5}{|c}{Linear; CI: 95\%}\\
\hline
UNIF & 0.872 & 0.872 & 0.888 & 0.914 & 0.890& 0.916 & 0.928 & 0.944 & 0.954 & 0.932 \\
LEV & 0.902 & 0.885 & 0.925 & 0.926 & 0.895 & 0.937 & 0.936 & 0.958 & 0.969 & 0.942 \\
GRAD & 0.881 & 0.908 & 0.892 & 0.913 & 0.887 & 0.931 & 0.960 & 0.950 & 0.957 & 0.936 \\
Hessian & 0.886 & 0.876 & 0.891 & 0.910 & 0.904 & 0.934 & 0.938 & 0.941 & 0.963 & 0.946 \\
\hline
& \multicolumn{5}{c}{Logistic; CI: 90\%} & \multicolumn{5}{|c}{Logistic; CI: 95\%} \\
\hline
UNIF & 0.904 & 0.874 & 0.892 & 0.896 & 0.896 & 0.958 & 0.944 & 0.938 & 0.952 & 0.950 \\
GRAD & 0.906 & 0.912 & 0.912 & 0.914 & 0.894 & 0.953 & 0.954 & 0.954 & 0.958 & 0.950 \\
Hessian & 0.889 & 0.917 & 0.890 & 0.918 & 0.895 & 0.927 & 0.972 & 0.967 & 0.955 & 0.943 \\
\hline
\end{tabular}
\end{center}
\label{Table-CI}
\end{table*}

{\bf MSE performance.}  
We plot the logarithm of $\text{MSE}(\hat{\theta}_n)$ with respect to $\log(n)$ for linear/logistic regressions in Figure \ref{MSE-rate} . Note LEV performs badly (worse than UNIF) in logistic regression, so LEV  is not used in logistic regression.
From the figure, we have two observations. First, the slopes of $\log(\text{MSE})$ with respect to $\log(n)$ are very close to 1. This means  that the decreasing rate of MSE is $O(n^{-1})$ and that the subsampled optimization can approximates the large-scale optimization well. Second, Hessian gets the best performance among four sampling methods. It verifies the efficiency of our Hessian-based sampling.

{\bf The performance of the MSE approximation.}  
We empirically investigate the performance of the MSE approximation. We calculate 
$\text{ratio}=\text{AMSE}(\hat{\theta}_n)/\text{MSE}(\hat{\theta}_n)$, and report them in Table \ref{Table-AMSEratio}. We see that $\text{AMSE}(\hat{\theta}_n)$ is very close to $\text{MSE}(\hat{\theta}_n)$ and their gap is small even when the sampling fraction equals 0.005. The results show that our MSE approximation behaves as expected in Theorem \ref{them-MSE}, that is, $\text{AMSE}(\hat{\theta}_n)$ approximates $\text{MSE}(\hat{\theta}_n)$ very well. 

{\bf The performance of MSE estimators.}  
We empirically investigate the performance of MSE estimators. 
 Table \ref{Table-MSEratio} summarizes the performance of MSE estimators $\text{mse}(\hat{\theta}_n)$ as ratios of means of $\text{mse}(\hat{\theta}_n)$ and $\text{MSE}(\hat{\theta}_n)$. We see that the means of $\text{mse}(\hat{\theta}_n)$ are very close to $\text{MSE}(\hat{\theta}_n)$ and their gap is small even when the sampling ratio equals 0.005. The results show that the MSE estimators behave as expected in our results, that is, the MSE estimators can well estimate the MSE.

{\bf Confidence region.} 
We calculate the frequency of $\hat{\theta}_N$ being in the confidence set  $\Im$ defined in Eqn.(\ref{CI}) among one thousand replicated samples, and report the results in Table \ref{Table-CI},
where the confidence levels $q$ are set among $\{0.9, 0.95\}$.  
The table shows that the frequencies being in the confidence set are very close to the confidence levels $q$ respectively.  
Clearly, it indicates that our constructed confidence region is sufficiently accurate. 

\section{Summary and remark}
In this paper, we investigate sampling-based approximation for optimization on large-scale data.  We propose the subsampled optimization for approximately solving it. Different with analysis on convergence rates, we analyze its asymptotic properties as statistical guarantees. Furthermore, we rigorously derive the MSE approximation with the approximation error of the order $O(n^{-3/2})$, and provide an MSE estimator whose bias is corrected to the same order. Due to this approximately unbiased MSE estimator and asymptotic normality we proved, we construct a confidence region for the aim of interest: the exact solution of large-scale optimization.
By applying our results, we also provide an optimal sampling, the Hessian-based sampling.

In this paper, we focus on large-scale optimization in finite dimension, and require the loss function to be twice-differentiable. It limits the class of the optimization. So how to extend our theoretical analysis to optimization in infinite dimension, and relax the assumptions to apply a boarder class of optimization problems is a worthy study in the future.
\bibliography{survey-sampling,optimalref}

\appendix
\setcounter{equation}{0}
\renewcommand{\theequation}{A.\arabic{equation}}

\section{Proving Theorem \ref{them:asymptotic}}


From the definition of $\hat{\theta}_n$, which is a zero of $\nabla^{\top} F_S(\theta)$, that is $\nabla^{\top} F_S(\hat{\theta}_n)=0$.
Taking a Taylor series expansion of $\nabla^{\top} F_S(\hat{\theta}_n)$ around $\hat{\theta}_N$.
There exists $\tilde{\theta}_{(j)}$ lies between $\hat{\theta}_n$ and $\hat{\theta}_N$ such that
\begin{align}\label{taylor}
0=&\nabla^{\top} F_S(\hat{\theta}_n)=\nabla^{\top} F_S(\hat{\theta}_N)+\left(\frac{\partial^2 F_S(\tilde{\theta}_{(j)})}{\partial \theta \partial \theta_j}(\hat{\theta}_n-\hat{\theta}_N)\right)_{1\leq j\leq d}\notag\\
=&\nabla^{\top} F_S(\hat{\theta}_N)+\nabla^2 F_N(\hat{\theta}_N)(\hat{\theta}_n-\hat{\theta}_N)
+\left[\nabla^2 F_S(\hat{\theta}_N)-\nabla^2 F_N(\hat{\theta}_N)\right](\hat{\theta}_n-\hat{\theta}_N)\notag\\
&+\left[\nabla^2 \tilde{F}_S-\nabla^2 F_S(\hat{\theta}_N)\right](\hat{\theta}_n-\hat{\theta}_N),
\end{align}
where we use the notation $\nabla^2 \tilde{F}_S$ denote such matrix that the $j$th row is $\frac{\partial^2 F_S(\tilde{\theta}_{(j)})}{\partial \theta \partial \theta_j}$.  
Define $\Sigma=\nabla^2 F_N(\hat{\theta}_N)$, 
$$Q_1=\nabla^2 F_S(\hat{\theta}_N)-\nabla^2 F_N(\hat{\theta}_N) \text{ , and  }Q_2=\nabla^2 \tilde{F}_S-\nabla^2 F_S(\hat{\theta}_N).$$ From Assumption B, $\nabla^2 F_N(\hat{\theta}_N)\geq \lambda I$, so $\Sigma$ is inversable. Therefore, Eqn.(\ref{taylor}) equals that
\begin{equation}\label{3-terms-P}
\hat{\theta}_n-\hat{\theta}_N=-\Sigma^{-1}\nabla^{\top} F_S(\hat{\theta}_N)-\Sigma^{-1}(Q_1+Q_2)(\hat{\theta}_n-\hat{\theta}_N).
\end{equation}
Define the events  
\begin{align*}
&\mathcal{E}_l:=\{\frac{1}{N}|\sum\limits_{i\in S}\frac{1}{\pi_i}L(x_i)-\sum\limits_{i=1}^NL(x_i)|\leq L\}, \text{ for contant } L\\
&\mathcal{E}_F:=\{\|\nabla^2 F_S(\theta)-\nabla^2 F_U(\theta)\|_2\leq c\lambda\} \text{ for }\theta\in \Theta, 0<c<1,
\end{align*} 
and $\mathcal{E}=\mathcal{E}_l\bigcup\mathcal{E}_F$, and denote $\mathcal{E}^c$ as the complement set of $\mathcal{E}$.
From combing Lemmas \ref{e1-c} and \ref{e2-c}, we have the claim:  
\begin{align}
&Pr(\mathcal{E}^c)=O(n^{-1}).\label{claim-ec}
\end{align} 
We  also prove the following claims:  
\begin{align}
&\text{E}[\|\Sigma^{-1}\nabla^{\top} F_S(\hat{\theta}_N)\|^2]=O(n^{-1}),\label{claim-main}\\
&\text{E}[\|\Sigma^{-1}(Q_1+Q_2)\|_2^2|\mathcal{E}]=O(n^{-1}),\label{claim-P}\\
&\text{E}[\|\hat{\theta}_n-\hat{\theta}_N\|^4|\mathcal{E}]=O(n^{-2}).\label{claim-theta}
\end{align} 
where the claim (\ref{claim-main}) is from Lemma \ref{FS-8th} and the inequality that $$\text{E}[\|\Sigma^{-1}\nabla^{\top} F_S(\hat{\theta}_N)\|^2]\leq \lambda^{-2}\text{E}[\|\nabla F_S(\hat{\theta}_N)\|^2];$$  
the claim (\ref{claim-P}) is proved in Lemmas \ref{P4} and \ref{Q2-4}. Now we prove the claim (\ref{claim-theta}).
From Lemma \ref{delta},  
\begin{equation*}
\text{E}[\|\hat{\theta}_n-\hat{\theta}_N\|^4|\mathcal{E}]\leq \frac{16}{\lambda^4(1-c)^4}\text{E}\|\nabla F_S(\hat{\theta}_N)\|^4=O(n^{-2}),
\end{equation*}
where the last equality is based on $\text{E}\|\nabla F_S(\hat{\theta}_N)\|^4=O(n^{-2})$ from Assumption D by applying Lemma \ref{FS-8th}.

After having the claims above, we shall prove the Eqns.(\ref{unbiasedness}) \& (\ref{consistency}).
From Eqn.(\ref{3-terms-P}), 
\begin{align}\label{unbiase-1}
\|\text{E}(\hat{\theta}_n-\hat{\theta}_N|\mathcal{E})\|&=\|\text{E}[\Sigma^{-1}(Q_1+Q_2)(\hat{\theta}_n-\hat{\theta}_N)|\mathcal{E}]\|\leq \text{E}[\|\Sigma^{-1}(Q_1+Q_2)(\hat{\theta}_n-\hat{\theta}_N)\||\mathcal{E}]\notag\\
&\leq (\text{E}[\|\Sigma^{-1}(Q_1+Q_2)\|_2^2|\mathcal{E}]E[\|\hat{\theta}_n-\hat{\theta}_N\|^2|\mathcal{E}])^{1/2}\notag\\
&= [O(n^{-1})O(n^{-1})]^{1/2}=O(n^{-1}),
\end{align}
where 1st step is from taking expectation on two sides of Eqn.(\ref{3-terms-P}), 2nd step is from the Jensen's inequality, 3rd step is from Cauchy-Schwartz inequality, 4th step is from two results that $\text{E}[\|\Sigma^{-1}(Q_1+Q_2)\|_2^2|\mathcal{E}]=O(n^{-1})$ in the claim (\ref{claim-P}) and $\text{E}[\|\hat{\theta}_n-\hat{\theta}_N\|^2|\mathcal{E}]=O(n^{-1})$ from the claim (\ref{claim-theta}).
From Eqn.(\ref{unbiase-1}), 
\begin{align}\label{unbiase-2}
\|\text{E}(\hat{\theta}_n)-\hat{\theta}_N\|\leq \|\text{E}[\hat{\theta}_n-\hat{\theta}_N|\mathcal{E}]\|+Pr(\mathcal{E}^c)R=O(n^{-1}),
\end{align}
where the last equality is from the claim (\ref{claim-ec}). Thus, Eqn.(\ref{unbiasedness}) is proved.

Similar to Eqn.(\ref{unbiase-2}), 
$$\text{E}\|\hat{\theta}_n-\hat{\theta}_N\|^2\leq \text{E}[\|\hat{\theta}_n-\hat{\theta}_N\|^2|\mathcal{E}]+Pr(\mathcal{E}^c)R^2=O(n^{-1}),$$
where the 1st inequality is from Assumption A, the 2nd inequality is combing the claim (\ref{claim-ec}) and $\text{E}[\|\hat{\theta}_n-\hat{\theta}_N\|^2|\mathcal{E}]=O(n^{-1})$ from the claim (\ref{claim-theta}).
Thus, Eqn.(\ref{consistency}) is proved.

Lastly, we shall prove the asymptotic normality. 
Combing Eqns.(\ref{3-terms-P}) \& (\ref{E1}), we have that
\begin{equation}\label{asy-main}
\hat{\theta}_n-\hat{\theta}_N=-\Sigma^{-1}\nabla F_S(\hat{\theta}_N)+O_p(n^{-1}).
\end{equation}
Now we investigate the term $\Sigma^{-1}\nabla F_S(\hat{\theta}_N)$.
Following the idea of \cite{Hansen:43}, we define an independent random vector sequence $\{\zeta_j\}_{j=1}^n$ such that each vector
$\zeta$ takes the value among $\{\Sigma^{-1}\frac{1}{N\pi_i}\nabla f(\hat{\theta}_N;x_i)\}_{i=1}^N$, and 
$$Pr\left(\zeta=\Sigma^{-1}\frac{1}{N\pi_i}\nabla f(\hat{\theta}_N;x_i)\right)=\pi_i, \ \ i=1, \cdots, N.$$
 From the definition, 
$E(\zeta)=\Sigma^{-1}\nabla F_U(\hat{\theta}_N)=0$ and $Var(\zeta)=\text{AMSE}(\hat{\theta}_n)$.
Note that, from the process of sampling with replacement,
\begin{equation}\label{eq:14}
\Sigma^{-1}\nabla F_S(\hat{\theta}_N)
  \equiv\frac{1}{n}\sum_{j=1}^n\zeta_j
\end{equation}
Given the large-scale dataset $F_U$, $\{\zeta_j\}_{j=1}^n$ are i.i.d, with mean $\mathbf{0}$ and variance $\text{AMSE}(\hat{\theta}_n)$. 
Meanwhile, for every $\gamma>0$, 
\begin{align}\label{LF-1}
   &\sum_{j=1}^n \mathrm{E}\{\|n^{-1/2}\zeta_j\|^2
    I(\|\zeta_j\|>n^{1/2}\gamma)|F_U\}\notag\\
    \le&\frac{1}{n^2\gamma^2}
    \sum_{j=1}^n\mathrm{E}\{\|\zeta_j\|^4
    I(\|\zeta_j\|>n^{1/2}\gamma)|F_U\}
  \le\frac{1}{n^2\gamma^2}
    \sum_{j=1}^n \mathrm{E}(\|\zeta_j\|^4|F_U)\notag\\
  =&\frac{1}{n\gamma^2}\frac{1}{N^4}
    \sum_{i=1}^N\frac{1}{\pi_i^3}\|\Sigma^{-1}\nabla f(\hat{\theta}_N;x_i)\|^4\notag\\
    =&\frac{1}{n\gamma^2}O(1)=o(1),
\end{align}
where the last but one step is from Assumption D. 
This equation (\ref{LF-1})  and the claim (\ref{claim-main}) show that the Lindeberg-Feller conditions are satisfied in probability.
Thus, by the Lindeberg-Feller central limit theorem (Proposition 2.27 of \cite{Vaart:98}), conditionally on $F_U$,
\begin{equation}\label{asy-main-norm}
\text{AMSE}(\hat{\theta}_n)^{-1/2}\left(\Sigma^{-1}\nabla F_S(\hat{\theta}_N)\right)
  \rightarrow N(0,I)  \text{,  in distribution.}
\end{equation}
Therefore, combing Eqns.(\ref{asy-main}) and (\ref{asy-main-norm}), Eqn.(\ref{normal}) is proved.

\section{Proving Theorem \ref{them-MSE}}

The way for proving Theorem \ref{them-MSE} is to analyze $\hat{\theta}_n-\hat{\theta}_N$ by carefully applying the Taylor expansion method as we prove Theorem \ref{them:asymptotic}.
From Eqn.(\ref{3-terms-P}), 
thereby, 
\begin{align}\label{mse-E12}
\text{E}(\hat{\theta}_n-\hat{\theta}_N)(\hat{\theta}_n-\hat{\theta}_N)^T=&\text{E}\left[(\Sigma^{-1}\nabla F_S(\hat{\theta}_N))(\Sigma^{-1}\nabla F_S(\hat{\theta}_N))^T\right]\notag\\
&+\text{E}[(\Sigma^{-1}(Q_1+Q_2)(\hat{\theta}_n-\hat{\theta}_N))(\Sigma^{-1}(Q_1+Q_2)(\hat{\theta}_n-\hat{\theta}_N))^T]\notag\\
&+\text{E}[(\Sigma^{-1}\nabla F_S(\hat{\theta}_N))(\Sigma^{-1}(Q_1+Q_2)(\hat{\theta}_n-\hat{\theta}_N))^T]\notag\\
&+\text{E}[(\Sigma^{-1}\nabla F_S(\hat{\theta}_N))^T(\Sigma^{-1}(Q_1+Q_2)(\hat{\theta}_n-\hat{\theta}_N))]\notag\\
=:&\text{E}\left[(\Sigma^{-1}\nabla F_S(\hat{\theta}_N))(\Sigma^{-1}\nabla F_S(\hat{\theta}_N))^T\right]+E_{11}+E_{12}+E_{21}.
\end{align}
Roughly speaking, we hope the last three terms in the expression (\ref{mse-E12}) to be of smaller order than the first term. We now formalize this intuition.
Let 
\begin{align*}
&E_1=\text{E}[\|\Sigma^{-1}(Q_1+Q_2)(\hat{\theta}_n-\hat{\theta}_N)\|^2]\notag\\
&E_2=\text{E}[\|\Sigma^{-1}\nabla F_S(\hat{\theta}_N)\|\|\Sigma^{-1}(Q_1+Q_2)(\hat{\theta}_n-\hat{\theta}_N)\|].
\end{align*}
Since the elements of $E_{11}$ are bounded by $E_1$, and the elements of $E_{12}$ and $E_{21}$ are bounded by $E_2$, it is sufficient to give the bound on $E_1$ and $E_2$ in order to bound the elements of $E_{11}$, $E_{12}$ and $E_{21}$.
Notice that
\begin{align}\label{Ptheta}
&\text{E}[\|\Sigma^{-1}(Q_1+Q_2)(\hat{\theta}_n-\hat{\theta}_N)\|^2|\mathcal{E}]
\leq 2\|\Sigma^{-1}\|_2^2E[(\|Q_1\|_2^2+\|Q_2\|_2^2)\|\hat{\theta}_n-\hat{\theta}_N\|^2|\mathcal{E}]\notag\\
\leq& 2\|\Sigma^{-1}\|_2^2(\text{E}[\|Q_1\|_2^4|\mathcal{E}])^{1/2}+(\text{E}[\|Q_2\|_2^4|\mathcal{E}])^{1/2}](\text{E}[\|\hat{\theta}_n-\hat{\theta}_N\|^4|\mathcal{E}])^{1/2}\notag\\
\leq& 2(1+O(n^{-1}))\|\Sigma^{-1}\|_2^2(\text{E}\|Q_1\|_2^4)^{1/2}+(\text{E}[\|Q_2\|_2^4|\mathcal{E}])^{1/2}](\text{E}[\|\hat{\theta}_n-\hat{\theta}_N\|^4|\mathcal{E}])^{1/2}\notag\\
=&O(n^{-2}),
\end{align}
where the 1st step is from matrix norm properties,  the 2nd step is from Cauchy-Schwartz inequality, the 3rd step is noting $\text{E}[\|Q_1\|_2^4|\mathcal{E}]\leq (1+O(n^{-1}))\text{E}[\|Q_1\|_2^4]$ from Lemma \ref{Conditonal-E}, and  the last step is from the claim (\ref{claim-theta}), Lemmas \ref{P4} and \ref{Q2-4}.
From Eqn.(\ref{Ptheta}), the claim (\ref{claim-ec}) and Lemmas \ref{P4} and \ref{Q2-4},
\begin{equation}\label{E1}
E_1\leq \text{E}[\|\Sigma^{-1}(Q_1+Q_2)(\hat{\theta}_n-\hat{\theta}_N)\|^2|\mathcal{E}]+Pr(\mathcal{E}^c)R^2E\|\Sigma^{-1}(Q_1+Q_2)\|_2^2=O(n^{-2}).
\end{equation}
From the claim (\ref{claim-main}) and Eqn.(\ref{E1}),
\begin{equation}\label{E2}
E_2\leq \left(E[\|\Sigma^{-1}\nabla F_S(\hat{\theta}_N)\|^2]\text{E}[\|\Sigma^{-1}(Q_1+Q_2)(\hat{\theta}_n-\hat{\theta}_N)\|^2]\right)^{1/2}=O(n^{-3/2}).
\end{equation}
Thus, combing  Eqns.(\ref{E1}) and (\ref{E2}), the theorem is proved.

\section{Proof of Theorem \ref{them-mse}}
Define 
\begin{align*}
&\Delta_{\Sigma^{-1}}=\Sigma^{-1}-\hat{\Sigma}^{-1},\notag\\
&A_n=\frac{1}{N^2n^2}\sum_{i \in S}\frac{1}{\pi_i^2}\nabla f(\hat{\theta}_n;x_i)\nabla^{\top} f(\hat{\theta}_n;x_i) \text{, and}\notag\\
&A_N=\frac{1}{N^2n^2}\sum_{i \in S}\frac{1}{\pi_i^2}\nabla f(\hat{\theta}_N;x_i)\nabla^{\top} f(\hat{\theta}_N;x_i).
\end{align*}
Thereby, we have that
\begin{align}\label{mse}
mse&=\Sigma^{-1}A_n\Sigma^{-1} +\Delta_{\Sigma^{-1}}A_n\Delta_{\Sigma^{-1}}-\Sigma^{-1}A_n\Delta_{\Sigma^{-1}}-\Delta_{\Sigma^{-1}}A_n\Sigma^{-1}\notag\\
&=:\Sigma^{-1}A_n\Sigma^{-1}+mse_r,
\end{align}
where 
$mse_r$ denotes the remainder terms.

Firstly, we investigate the term $\Sigma^{-1}A_n\Sigma^{-1}$.  Note that 
\begin{equation}\label{mse-taylor}
\Sigma^{-1}A_n\Sigma^{-1}=\Sigma^{-1}A_N\Sigma^{-1}+\Sigma^{-1}(A_n-A_N)\Sigma^{-1}.
\end{equation}
For $\Sigma^{-1}A_N\Sigma^{-1}$,  $$\text{E}[\Sigma^{-1}A_N\Sigma^{-1}]=\Sigma^{-1}\left[\frac{1}{N^2n}\sum_{i \in U}\frac{1}{\pi_i}\nabla f(\hat{\theta}_N;x_i)\nabla^{\top} f(\hat{\theta}_N;x_i)\right]\Sigma^{-1}.$$
Now we investigate $\Sigma^{-1}(A_n-A_N)\Sigma^{-1}$. Let $l_{\Delta}(x_i)=\nabla f(\hat{\theta}_n;x_i)-\nabla f(\hat{\theta}_N;x_i)$, so we have 
$$A_n-A_N=\frac{1}{N^2n^2}\sum_{i \in S}\frac{1}{\pi_i^2}\left[\nabla f(\hat{\theta}_N;x_i)l_{\Delta}(x_i)^{\top}+l_{\Delta}(x_i)\nabla^{\top}f(\hat{\theta}_N;x_i)+l_{\Delta}(x_i)l_{\Delta}(x_i)^{\top}\right].$$
It follows that
\begin{align}\label{R0}
\text{E}\|A_n-A_N\|_F
\leq& \frac{1}{N^2n^2}\text{E}\left(\sum_{i \in S}\frac{1}{\pi_i^2}\left[2\|\nabla f(\hat{\theta}_N;x_i)\|\|l_{\Delta}(x_i)\|+\|l_{\Delta}(x_i)\|^2|\right]\right)\notag\\
\leq& \frac{1}{N^2n^2}\text{E}\left(\sum_{i \in S}\frac{1}{\pi_i^2}\left[2L(x_i)\|\nabla f(\hat{\theta}_N;x_i)\|\|\hat{\theta}_n-\hat{\theta}_N\|+L(x_i)^2\|\hat{\theta}_n-\hat{\theta}_N)\|^2\right]\right)\notag\\
\leq&\frac{1}{N^2n^2}\text{E}\left(\sum_{i \in S}\frac{1}{\pi_i^2}\left[2L(x_i)\|\nabla f(\hat{\theta}_N;x_i)\|\|\hat{\theta}_n-\hat{\theta}_N\|+L(x_i)^2\|\hat{\theta}_n-\hat{\theta}_N)\|^2\right]|\mathcal{E}\right)\notag\\
&\ \ +Pr(\mathcal{E}^c)\frac{1}{N^2n^2}E\left(\sum_{i \in S}\frac{1}{\pi_i^2}\left[2RL(x_i)\|\nabla f(\hat{\theta}_N;x_i)\|+R^2L(x_i)^2\right]\right)\notag\\
\leq&2\left[E\left(\frac{1}{N^2n^2}\sum_{i \in S}\frac{1}{\pi_i^2}L(x_i)\|\nabla f(\hat{\theta}_N;x_i)\||\mathcal{E}\right)^2E\left(\|\hat{\theta}_n-\hat{\theta}_N\|^2|\mathcal{E}\right)\right]^{1/2}\notag\\
&\ \ +\left[\text{E}\left(\frac{1}{N^2n^2}\sum_{i \in S}\frac{1}{\pi_i^2}L(x_i)^2|\mathcal{E}\right)^2E\left(\|\hat{\theta}_n-\hat{\theta}_N)\|^4|\mathcal{E}\right)\right]^{1/2}\notag\\
&\ \ +Pr(\mathcal{E}^c)\text{E}\left(\frac{1}{N^2n^2}\sum_{i \in S}\frac{1}{\pi_i^2}\left[2RL(x_i)\|\nabla f(\hat{\theta}_N;x_i)\|+R^2L(x_i)^2\right]\right),
\end{align}
where the second inequality is from Assumption C that $\|l_{\Delta}(x_i)\|\leq L(x_i)\|\hat{\theta}_n-\hat{\theta}_N\|$, and other steps are similar to proving Eqn.(\ref{E1})
From Eqn.(\ref{R0}), $Pr(\mathcal{E}^c)=O(n^{-1})$ in the claim (\ref{claim-ec}), $\text{E}[\|\hat{\theta}_n-\hat{\theta}_N\|^4|\mathcal{E}]=O(n^{-2})$ in the claim (\ref{claim-theta}) and Lemma \ref{Conditonal-E}, so we have that
\begin{align}\label{R}
\text{E}\|A_n-A_N\|_F
\leq&2\left[(1+O(n^{-1}))\text{E}\left(\frac{1}{N^2n^2}\sum_{i \in S}\frac{1}{\pi_i^2}L(x_i)\|\nabla f(\hat{\theta}_N;x_i)\|\right)^2O(n^{-1})\right]^{1/2}\notag\\
& \ \ +\left[(1+O(n^{-1}))\text{E}\left(\frac{1}{N^2n^2}\sum_{i \in S}\frac{1}{\pi_i^2}L(x_i)^2\right)^2O(n^{-2})\right]^{1/2}\notag\\
&\ \ +O(n^{-1})\frac{1}{N^2n}\sum_{i\in U}\frac{1}{\pi_i}\left[2RL(x_i)\|\nabla f(\hat{\theta}_N;x_i)\|+R^2L(x_i)^2\right]\notag\\
=&O(n^{-3/2}),
\end{align}
where the last equality is gotten from 
Assumption D by applying Lemme \ref{pij-est}.
Thus,  from Eqns.(\ref{mse-taylor}) \& (\ref{R}), we have that 
\begin{equation}\label{mse1-result}
\text{E}[\Sigma^{-1}A_n\Sigma^{-1}]=\text{E}(\hat{\theta}_n-\hat{\theta}_N)(\hat{\theta}_n-\hat{\theta}_N)^T+O(n^{-3/2}).
\end{equation}

Secondly, we turn to the term $mse_r$. To bound it, we investigate the terms $A_n$ and $\Delta_{\Sigma^{-1}}$ in the following. 
For the term $A_n$, 
\begin{align}\label{E-mse_r}
\text{E}\|A_n\|_F^2&\leq \text{E}\left[\frac{1}{N^2n^2}\sum_{i \in S}\frac{1}{\pi_i^2}\left\|\nabla f(\hat{\theta}_n;x_i)\right\|^2\right]^2\notag\\
&\leq \text{E}\left[\frac{2}{N^2n^2}\sum_{i \in S}\frac{1}{\pi_i^2}\left(\left\|\nabla f(\hat{\theta}_N;x_i)\right\|^2+\|l_{\Delta}(x_i)\|^2\right)\right]^2\notag\\
&\leq \text{E}\left[\frac{2}{N^2n^2}\sum_{i \in S}\frac{1}{\pi_i^2}\left(\left\|\nabla f(\hat{\theta}_N;x_i)\right\|^2+L(x_i)^2\|\hat{\theta}_n-\hat{\theta}_N\|^2\right)\right]^2\notag\\
&\leq \text{E}\left[\frac{4}{N^2n^2}\sum_{i \in S}\frac{1}{\pi_i^2}\left\|\nabla f(\hat{\theta}_N;x_i)\right\|^2\right]^2+\text{E}\left[\frac{2}{N^2n^2}\sum_{i \in S}\frac{1}{\pi_i^2}L(x_i)^2\|\hat{\theta}_n-\hat{\theta}_N\|^2\right]^2\notag\\
&\leq \text{E}\left[\frac{4}{N^2n^2}\sum_{i \in S}\frac{1}{\pi_i^2}\left\|\nabla f(\hat{\theta}_N;x_i)\right\|^2\right]^2+R^4\text{E}\left[\frac{2}{N^2n^2}\sum_{i \in S}\frac{1}{\pi_i^2}L(x_i)^2\right]^2,
\end{align}
where the last inequality is from Assumption A.
Notice that $\text{E}\left[\frac{1}{N^2n^2}\sum_{i \in S}\frac{1}{\pi_i^2}\left\|\nabla f(\hat{\theta}_N;x_i)\right\|^2\right]^2=O(n^{-2})$ and
$\text{E}\left[\frac{1}{N^2n^2}\sum_{i \in S}\frac{1}{\pi_i^2}L(x_i)^2\right]^2=O(n^{-2})$ from the Assumption D by using Lemma \ref{pij-est}. 
Thus, from Eqn.(\ref{E-mse_r}), we have that 
\begin{equation}\label{E-A}
\text{E}\|A_n\|_F^2=O(n^{-2}).
\end{equation}
We now investigate the term $\Delta_{\Sigma^{-1}}$. 
Since 
$$\Delta_{\Sigma^{-1}}=\Sigma^{-1}(\hat{\Sigma}-\Sigma)\hat{\Sigma}^{-1}, $$
we have that,  
\begin{equation}\label{DD}
\|\Delta_{\Sigma^{-1}}\|_2\leq \|\Sigma^{-1}\|_2\|\hat{\Sigma}-\Sigma\|_2\|\hat{\Sigma}^{-1}\|_2\leq \frac{1}{\lambda^2(1-c)}\|\hat{\Sigma}-\Sigma\|_2.
\end{equation}
where the last inequality is from Assumption B and Eqn.(\ref{convex-F}) given the event $\mathcal{E}_F$ holds. 
Define $Q_3=\frac{1}{Nn}\sum\limits_{i\in S}\frac{1}{\pi_i}\nabla^2 f(\hat{\theta}_n;x_i)-\frac{1}{Nn}\sum\limits_{i\in S}\frac{1}{\pi_i}\nabla^2 f(\hat{\theta}_N;x_i)$.
Notice that
\begin{align}\label{D}
E\|\hat{\Sigma}-\Sigma\|_2^2&=E\left\|N^{-1}n^{-1}\sum\limits_{i\in S}\pi_i^{-1}\nabla^2 f(\hat{\theta}_n;x_i)-N^{-1}\sum\limits_{i=1}^N\nabla^2 f(\hat{\theta}_N;x_i)\right\|_2^2\notag\\
&=E\|Q_3+Q_1\|_2^2\leq 2\text{E}\|Q_3\|_2^2+2\text{E}\|Q_1\|_2^2=O(n^{-1}),
\end{align}
where the last equality is from Lemmas \ref{P4} and \ref{Q3-4}. 
 Noting that 
 $\|\Delta_{\Sigma^{-1}}\|_2=\|\Sigma^{-1}-\hat{\Sigma}^{-1}\|_2\leq \|\Sigma^{-1}\|_2+\|\hat{\Sigma}^{-1}\|_2\leq C$ for some constant $C$, 
so we have that, for some constant $C$, 
\begin{align}\label{mse2}
[E\|mse_r\|_F]^2&\leq [d^{1/2}E\|mse_r\|_2]^2 \leq C[E\|A_n\Delta_{\Sigma^{-1}}\|_2]^2\notag\\
&\leq CE\|A_n\|_2^2E\|\Delta_{\Sigma^{-1}}\|_2^2
\leq CE\|A_n\|_F^2\left(\text{E}\|\hat{\Sigma}-\Sigma\|_2^2])\right)\notag\\
&=O(n^{-3}),
\end{align}
where the first inequality is from that, for the matrix $mse_r$ with rank $d$, $\|mse_r\|_F\leq d^{1/2}\|mse_r\|_2$ based on the inequality of matrix norms, 
 the second inequality is from that $\|\Delta_{\Sigma^{-1}}\|_2\leq C$ where $C$ is some constant,  the third inequality is from the Cauchy-Schwartz inequality, the fourth inequality is from Eqn.(\ref{DD}),  and the last step is from Eqns.(\ref{E-A}) and (\ref{D}).
From Eqn.(\ref{mse2}), $$E\|mse_r\|_F=O(n^{-3/2}).$$
Thus, Theorem \ref{them-mse} is proved.

\section{Lemmas}
\label{sec:lemma}

\begin{lemma}\label{pi-est}
A sample $\{a_i\}_{i\in S}$ of size $n$ is drawn from for a finite population $\{a_i\}_{i=1}^N$ according to the sampling probability $\{\pi_i\}_{i=1}^N$. If the condition that, for $k=2, 4$,
\begin{align}
&N^{-k}\sum\limits_{i=1}^N\frac{1}{\pi_i^{k-1}}a_i^k=O(1)\label{pi-a}
\end{align}
hold, 
then 
\begin{equation*}
E\left(N^{-1}n^{-1}\sum\limits_{i\in S}\frac{1}{\pi_i}a_i-N^{-1}\sum\limits_{i=1}^Na_i\right)^k=O(n^{-k/2}).
\end{equation*}
\end{lemma}

\begin{proof}
Define an independent random variable sequence $\{\eta_j\}_{j=1}^n$ such that each variable 
$\eta$ takes the value among $\{\frac{a_i}{\pi_i}$, $i=1, \cdots, N\}$, and 
$$Pr(\eta=\frac{a_i}{\pi_i})=\pi_i, \ \ i=1, \cdots, N.$$
Let $A=\sum\limits_{i=1}^Na_i$. From the definition, 
$E(\eta)=A$. 
Notice that, from the process of sampling with replacement, 
$$\frac{1}{n}\sum\limits_{i\in S}\frac{1}{\pi_i}a_i\equiv\frac{1}{n}\sum\limits_{j=1}^n\eta_j.$$
By direct calculation, 
\begin{equation*}
\frac{1}{n}\sum\limits_{i\in S}\frac{1}{\pi_i}a_i-\sum\limits_{i=1}^Na_i=\frac{1}{n}\sum\limits_{j=1}^n(\eta_j-A).
\end{equation*}
Here, we apply Rosenthal inequality to get the bound. 
Since $\left\{(\eta_j-A)/n\right\}_{j=1}^n$ are independent mean 0 random variables, by Rosenthal inequality, for some constant $C$,
\begin{align*}
\left(E\left[\frac{1}{N}\sum\limits_{j=1}^n\frac{\eta_j-A}{n}\right]^k\right)^{1/k}
\leq& \frac{C}{Nn}\left[\left(\sum\limits_{j=1}^nE(\eta_j-A)^2\right)^{1/2}+\left(\sum\limits_{j=1}^nE(\eta_j-A)^k\right)^{1/k}\right]\notag\\
= &\frac{C}{Nn}\left(n\sum\limits_{i=1}^N\frac{1}{\pi_i}a_i^2-nA^2\right)^{1/2}+\frac{C}{Nn}\left(n\sum\limits_{i=1}^N\pi_i(\frac{a_i}{\pi_i}-A)^k\right)^{1/k}\notag\\
=&O(n^{-1/2}) ,
\end{align*}
where the second step is from the definition of $\{\eta_j\}_{j=1}^n$, and the last step is from the condition Eqn.(\ref{pi-a}). Thus, 
Lemma \ref{pi-est} is proved.

\end{proof}

\begin{lemma}\label{pij-est}
A sample $\{a_i\}_{i\in S}$ is drawn with replacement from $\{a_i\}_{i=1}^N$ according to the sampling probability $\{\pi_i\}_{i=1}^N$.  If
\begin{equation}\label{MSE-order}
\frac{1}{N^4}\sum\limits_{i=1}^N\frac{a_i^4}{\pi_i^3}=O(1)
\end{equation}
hold, 
then 
\begin{equation*}
E\left(N^{-2}n^{-2}\sum\limits_{i\in S}\frac{1}{\pi_i^2}a_i^2\right)^2=O(n^{-2}).
\end{equation*}
\end{lemma}

\begin{proof}
From the definition of $\{\eta_j\}_{j=1}^n$ in proving Lemma \ref{pi-est}, 
$$(\sum\limits_{i\in S}\frac{1}{\pi_i^2}a_i^2)^2\equiv(\sum\limits_{j=1}^n\eta_j^2)^2\leq n\sum\limits_{j=1}^n\eta_j^4,$$
where the last inequality is from Cauchy-Schwartz inequality.
It follows that
\begin{align*}
E\left(N^{-2}n^{-2}\sum\limits_{i\in S}\frac{1}{\pi_i^2}a_i^2\right)^2
&=\frac{1}{n^4}\frac{1}{N^4}nE\left(\sum\limits_{j=1}^n\eta_j^4\right)=\frac{1}{n^2}\frac{1}{N^4}E\left(\eta_j^4\right)\notag\\
&=\frac{1}{n^2}\frac{1}{N^4}\sum\limits_{i=1}^N\frac{a_i^4}{\pi_i^3}=O(n^{-2}),
\end{align*}
where the second equality is from the independence among $\{\eta_j\}_{j=1}^n$, 
the third equality is from the fact $E(\eta_j^4)=\sum_{i=1}^n\frac{a_i^4}{\pi_i^3}$, and the last equality is from the condition Eqn.(\ref{MSE-order}). 
Thus, the lemma is proved.

\end{proof}


\begin{lemma}\label{e1-c}
Recall $\mathcal{E}_l:=\{\frac{1}{N}|\sum\limits_{i\in S}\frac{1}{\pi_i}L(x_i)-\sum\limits_{i=1}^NL(x_i)|\leq L\}$ for some constant $L>0$, and denote $\mathcal{E}_l^c$ as the complement set of $\mathcal{E}_l$.
If the following assumptions
\begin{align}
&\frac{1}{N^2}\sum\limits_{i=1}^N\frac{1}{\pi_i}(L(x_i))^2=O(n^{-1})\label{Assum-pi-2}
\end{align}
hold, then
$$Pr(\mathcal{E}_l^c)=O(n^{-1}).$$
\end{lemma}

\begin{proof}
From the definition of $\mathcal{E}$, 
\begin{equation}\label{Ec}
Pr(\mathcal{E}_l^c)\leq\frac{1}{L^2}E\left[\frac{1}{N}\sum\limits_{i\in S}\frac{1}{\pi_i}L(x_i)-\frac{1}{N}\sum_{i=1}^NL(x_i)\right]^2=O(n^{-1})\notag\\
\end{equation}
where the last inequality is from Eqn.(\ref{Assum-pi-2}) (Assumption D) by applying Lemma \ref{pi-est}.
\end{proof}

\begin{lemma}\label{P4}
Recall $Q_1=\nabla^2 F_S(\hat{\theta}_N)-\nabla^2 F_N(\hat{\theta}_N)$.
If the following assumptions 
\begin{align}
&\frac{1}{N^4}\sum\limits_{i=1}^N\frac{1}{\pi_i^3}\left\|\nabla^2 f(\hat{\theta}_N;x_i)\right\|_2^4=O(1).\label{C2}
\end{align}
holds, 
\begin{equation*}
E[\|Q_1\|_2^4]=O(n^{-2}).
\end{equation*}
\end{lemma}

\begin{proof}
Let $(j_1,j_2)$ element of $\nabla^2 f(\hat{\theta}_N;x_i)$ as $A_i^{j_1j_2}$. 
From the matrix norm inequality, 
\begin{align*}
&E\left[\left\|\nabla^2 F_S(\hat{\theta}_N)-\nabla^2 F_N(\hat{\theta}_N)\right\|_2^4\right]\notag\\
\leq&
E\left[\left\|\frac{1}{Nn}\sum\limits_{i\in S}\left(\frac{1}{\pi_i}\nabla^2 f(\hat{\theta}_N;x_i)-\sum\limits_{i=1}^N\nabla^2 f(\hat{\theta}_N;x_i)\right)\right\|_F^4\right].\notag\\
=&E\left[\sum\limits_{j_1,j_2=1}^d\left(\frac{1}{Nn}\sum\limits_{i\in S}\left(\frac{1}{\pi_i}A_i^{j_1j_2}-\sum\limits_{i=1}^NA_i^{j_1j_2}\right)\right)^2\right]^2=O(n^{-2}),
\end{align*}
where the last equality is gotten from the condition Eqn.(\ref{C2}) by applying  Lemma \ref{pij-est}.
Thus, the lemma is proved.

\end{proof}

\begin{lemma}\label{e2-c}
Recall $\mathcal{E}_F:=\{\|\nabla^2 F_S(\theta)-\nabla^2 F_U(\theta)\|_2\leq c\lambda\text{ for }\theta\in \Theta, 0<c<1\}$ , and denote $\mathcal{E}_F^c$ as the complement set of $\mathcal{E}_F$.
If the assumption (\ref{C2}) holds, 
then
$$Pr(\mathcal{E}_F^c)=O(n^{-1}).$$
\end{lemma}

\begin{proof}
Notice that
\begin{align}\label{C2-theta}
\frac{1}{N^4}\sum\limits_{i=1}^N\frac{1}{\pi_i^3}\left\|\nabla^2 f(\theta;x_i)\right\|_2^4&=
\frac{1}{N^4}\sum\limits_{i=1}^N\frac{1}{\pi_i^3}\left\|\nabla^2 [f(\hat{\theta}_N;x_i)+f(\theta;x_i)-f(\hat{\theta}_N;x_i)]\right\|_2^4\notag\\
&\leq\frac{2}{N^4}\sum\limits_{i=1}^N\frac{1}{\pi_i^3}\left\|\nabla^2 f(\hat{\theta}_N;x_i)\right\|_2^4+\frac{2}{N^4}\sum\limits_{i=1}^N\frac{1}{\pi_i^3}\left\|\nabla^2f(\theta;x_i)-\nabla^2f(\hat{\theta}_N;x_i)\right\|_2^4\notag\\
&\leq\frac{2}{N^4}\sum\limits_{i=1}^N\frac{1}{\pi_i^3}\left\|\nabla^2 f(\hat{\theta}_N;x_i)\right\|_2^4+\frac{2}{N^4}\sum\limits_{i=1}^N\frac{1}{\pi_i^3}L(x_i)^4\|\theta-\theta_N\|^4\notag\\
&=O(1).
\end{align}
where the last step is from Eqn.(\ref{C2}) and Assumptions A \& D.
From the definition of $\mathcal{E}_F$, 
\begin{equation}\label{bc}
Pr(\mathcal{E}_F^c)\leq\frac{1}{c^2}E\left[\left\|\nabla^2 F_S(\theta)-\nabla^2 F_N(\theta)\right\|_2^2\right]=O(n^{-1})\notag\\
\end{equation}
where the last inequality is from applying Lemma \ref{P4}.
\end{proof}

\begin{lemma}\label{delta}
Under the conditions that the function $f(\theta; x_i)$ is $\lambda$-strongly convex over $\Theta$, if the event $\mathcal{E}_F$ holds, then
\begin{equation*}
\|\hat{\theta}_n-\hat{\theta}_N\|\leq \frac{2}{\lambda (1-c)}\left\|\nabla F_S(\hat{\theta}_N)\right\|.
\end{equation*}
\end{lemma}

\begin{proof}
From Assumption B, given the event $\mathcal{E}_F$ holds (see Lemma \ref{e2-c}), 
the function $F_S$ is $\lambda (1-c)$-strongly convex over $\Theta$, that is, for $\theta\in \Theta$,
\begin{align}\label{convex-F}
\nabla^2 F_S(\theta)&=\nabla^2 F_U(\theta)-[\nabla^2 F_S(\theta)-\nabla^2 F_U(\theta)]\geq F_U(\theta)-  \|\nabla^2 F_S(\theta)-\nabla^2 F_U(\theta)\|_2I\notag\\
&\geq \lambda(1-c)I.
\end{align}
For any $\hat{\theta}_n \in \Theta$, taking a Taylor expansion around $\hat{\theta}_N$, 
we have, for a $\theta_1$ between $\hat{\theta}_n$ and $\hat{\theta}_N$,
\begin{align}\label{F-taylor}
F_S(\hat{\theta}_n)&= F_S(\hat{\theta}_N)+\nabla^{\top} F_S(\hat{\theta}_N)(\hat{\theta}_n-\hat{\theta}_N)+2^{-1}(\hat{\theta}_n-\hat{\theta}_N)^T\nabla^2 F_S(\theta_1)(\hat{\theta}_n-\hat{\theta}_N)\notag\\
&\geq  F_S(\hat{\theta}_N)+\nabla^{\top} F_S(\hat{\theta}_N)(\hat{\theta}_n-\hat{\theta}_N)+2^{-1}\lambda (1-c)\|\hat{\theta}_n-\hat{\theta}_N\|^2\notag\\
&\geq  F_S(\hat{\theta}_N)-\|\nabla F_S(\hat{\theta}_N)\|\|\hat{\theta}_n-\hat{\theta}_N\|+2^{-1}\lambda (1-c)\|\hat{\theta}_n-\hat{\theta}_N\|^2.
\end{align}
where the inequality in the second step is from Eqn.(\ref{convex-F}) and the last inequality is from Cauchy-Schwartz inequality.
Since $F_S(\hat{\theta}_n)\leq F_S(\hat{\theta}_N)$ which is based on the fact $\hat{\theta}_n$ is the minimizer of the function $F_S$, we have from Eqn.(\ref{F-taylor}) that,
if the event $\mathcal{E}_F$ holds,
$$\|\hat{\theta}_n-\hat{\theta}_N\|\leq \frac{2}{\lambda (1-c)}\left\|\nabla F_S(\hat{\theta}_N)\right\|.$$
\end{proof}


\begin{lemma}\label{Q2-4}
Recall 
$Q_2=\nabla^2 \tilde{F}_S-\nabla^2 F_S(\hat{\theta}_N)$.
We have that 
\begin{align}
&E[\|Q_2\|_2^2]=O(n^{-1}),\label{Q-2}\\
&E[\|Q_2\|_2^4|\mathcal{E}]=O(n^{-2}),\label{Q-4}
\end{align}

\end{lemma}

\begin{proof}
From the smoothness condition Assumption C, 
\begin{equation*}
\left\|\frac{\partial^2 }{\partial \theta \partial \theta_j}F_S(\tilde{\theta}_{(j)})-\frac{\partial^2 }{\partial \theta \partial \theta_j}F_S(\hat{\theta}_N)\right\|\leq \frac{1}{Nn}\sum\limits_{i\in S}\frac{1}{\pi_i} L(x_i)\|\tilde{\theta}_{(j)}-\hat{\theta}_N\|\leq \frac{1}{Nn}\sum\limits_{i\in S}\frac{1}{\pi_i} L(x_i)\|\hat{\theta}_n-\hat{\theta}_N\|
\end{equation*}
From the notation of the spectral norm, the above equation imply that 
\begin{align}\label{Q2-2in}
\|Q_2\|_2\leq & \|Q_2\|_F=\left(\sum\limits_{j=1}^d\left\|\frac{\partial^2 }{\partial \theta \partial \theta_j}F_S(\tilde{\theta}_{(j)})-\frac{\partial^2 }{\partial \theta \partial \theta_j}F_S(\hat{\theta}_N)\right\|^2\right)^{1/2}\notag\\
\leq & \sqrt{d}\frac{1}{Nn}\sum\limits_{i\in S}\frac{1}{\pi_i} L(x_i)\|\hat{\theta}_n-\hat{\theta}_N\|.
\end{align}
Firstly we investigate the bound of $E[\|Q_2\|_2^2]$. 
From Eqn.(\ref{Q2-2in}), 
\begin{align*}
E[\|Q_2\|_2^2]&\leq \text{E}\left(d\left[\frac{1}{Nn}\sum\limits_{i\in S}\frac{1}{\pi_i}L(x_i)\|\hat{\theta}_n-\hat{\theta}_N\|\right]^2\right)\notag\\
&\leq \text{E}\left(d\left[\frac{1}{Nn}\sum\limits_{i\in S}\frac{1}{\pi_i}L(x_i)\|\hat{\theta}_n-\hat{\theta}_N\|\right]^2|\mathcal{E}\right)+Pr(\mathcal{E}^c)dR^2E\left[\frac{1}{Nn}\sum\limits_{i\in S}\frac{1}{\pi_i}L(x_i)\right]\notag\\
&\leq d\left(N^{-1}\sum\limits_{i=1}^NL(x_i)+L\right)E[\|\hat{\theta}_n-\hat{\theta}_N\|^2|\mathcal{E}]+O(n^{-1})=O(n^{-1}),
\end{align*}
where the third inequality if from Lemma \ref{e1-c}, and the last equality is from $E[\|\hat{\theta}_n-\hat{\theta}_N\|^2|\mathcal{E}]=O(n^{-1})$.
Thus, Eqn.(\ref{Q-2}) in Lemma \ref{Q2-4} is proved.

Secondly we investigate the bound of $E[\|Q_2\|_2^4|\mathcal{E}]$. 
From Eqn.(\ref{Q2-2in}), we have that,
\begin{align*}
E[\|Q_2\|_2^4|\mathcal{E}]\leq&\text{E}\left(d\left[\frac{1}{Nn}\sum\limits_{i\in S}\frac{1}{\pi_i}L(x_i)\|\hat{\theta}_n-\hat{\theta}_N\|\right]^4|\mathcal{E}\right)\leq d\left(\frac{1}{Nn}\sum\limits_{i=1}^NL(x_i)+L\right)E\left[\|\hat{\theta}_n-\hat{\theta}_N\||\mathcal{E}\right]^4\notag\\
=&O(n^{-2}),
\end{align*}
where the last step is from $E[\|\hat{\theta}_n-\hat{\theta}_N\|^4|\mathcal{E}]=O(n^{-2})$ which is obtain from Lemmas \ref{delta} and \ref{FS-8th}.
Thus, Equation Eqn.(\ref{Q-4}) in Lemma \ref{Q2-4} is proved.

\end{proof}

\begin{lemma}\label{Q3-4}
Recall 
$Q_3=\frac{1}{Nn}\sum\limits_{i\in S}\frac{1}{\pi_i}\nabla^2 f(\hat{\theta}_n;x_i)-\frac{1}{Nn}\sum\limits_{i\in S}\frac{1}{\pi_i}\nabla^2 f(\hat{\theta}_N;x_i)$.
We have that 
\begin{align}
&E[\|Q_3\|_2^2]=O(n^{-4}).\label{Q2-8}
\end{align}
\end{lemma}

\begin{proof}
We investigate the bound of $E[\|Q_3\|_2^2]$. Following the proving of Eqn.(\ref{Q-2}),
we have that 
\begin{align*}
E[\|Q_3\|_2^2] &\leq E[\|Q_3\|_2^2|\mathcal{E}]+Pr(\mathcal{E}^c)R^2E\left[\frac{1}{Nn}\sum\limits_{i\in S}\pi_i^{-1}L(x_i)\right]\notag\\
&\leq E\left[\left(\frac{1}{Nn}\sum\limits_{i\in S}\pi_i^{-1}L(x_i)\|\hat{\theta}_n-\hat{\theta}_N\|\right)^2|\mathcal{E}\right]+O(n^{-1})\notag\\
&\leq \left[E\left((\frac{1}{Nn}\sum\limits_{i\in S}\pi_i^{-1}L(x_i))^2|\mathcal{E}\right)E(\|\hat{\theta}_n-\hat{\theta}_N\|^{4}|\mathcal{E})\right]^{1/2}+O(n^{-1})=O(n^{-1}),
\end{align*}
where the second inequality is from the smoothness condition Assumption C and the claim (\ref{claim-ec}), the third inequality is from the Cauchy-Schwartz inequality, and the last equality is from Lemmas \ref{delta} and \ref{FS-8th}. Thus, Eqn.(\ref{Q2-8}) is proved.

\end{proof}


\begin{lemma}\label{FS-8th}
 If the following assumptions 
\begin{align}
&\frac{1}{N^{4}}\sum\limits_{i=1}^N\frac{1}{\pi_i^{3}}\left\|\nabla f(\hat{\theta}_N;x_i)\right\|^{4}=O(n^{-2}).\label{C4}
\end{align}
holds, 
\begin{equation*}
E\left[\left\|\nabla F_S(\hat{\theta}_N)\right\|^{4}\right]=O(n^{-2})
\end{equation*}
\end{lemma}

\begin{proof}
Note that
\begin{align}\label{Fs-elements-step}
E\left\|\nabla F_S(\hat{\theta}_N)\right\|^{4}&=E\left[\sum\limits_{k=1}^d\left(\frac{1}{N^2n^2}\sum\limits_{i\in S}\frac{1}{\pi_i}(\nabla f(\hat{\theta}_N;x_i))_k\right)^2\right]^2\notag\\
&\leq C\sum\limits_{k=1}^dE\left(\frac{1}{N^2n^2}\sum\limits_{i\in S}\frac{1}{\pi_i}\left(\nabla f(\hat{\theta}_N;x_i)\right)_k\right)^{4}
\end{align}
Applying Lemma \ref{pij-est}, under the condition (\ref{C4}) we have that
\begin{equation}\label{A}
E\left[\frac{1}{N^2n^2}\sum\limits_{i\in S}\frac{1}{\pi_i}\left(\nabla f(\hat{\theta}_N;x_i)\right)_k\right]^{4}=O(n^{-2})
\end{equation}
Combing Eqns. (\ref{Fs-elements-step}) and (\ref{A}), the result is proved.
\end{proof}

\begin{lemma}\label{Conditonal-E}
For a random variable $Z\geq 0$, 
we have that 
\begin{align*}
\text{E}(Z|\mathcal{E})\leq (1+O(n^{-1}))\text{E}(Z).
\end{align*}
\end{lemma}

\begin{proof}
\begin{align*}
\text{E}(Z)=Pr(\mathcal{E})\text{E}(Z|\mathcal{E})+Pr(\mathcal{E}^c)\text{E}(Z|\mathcal{E}^c)\geq Pr(\mathcal{E})\text{E}(Z|\mathcal{E}).
\end{align*}
It follows that 
$$\text{E}(Z|\mathcal{E})\leq \frac{1}{Pr(\mathcal{E})}\text{E}(Z)=(1-O(n^{-1}))^{-1}\text{E}(Z)=(1+O(n^{-1}))\text{E}(Z).$$
Thus, the result is proved.
\end{proof}

\section{Additional simulation: comparison with equal weighting} 
\label{Addit-simulation}
To investigate the necessity of inverse probability weighting used in the subsampled optimization, 
we solve the direct optimization procedure from the subsampled data points with equal weighting and obtain the solution $\tilde{\theta}_n$.
We empirically compare $\hat{\theta}_n$ with $\tilde{\theta}_n$.
Here we consider misspecified models to better compare them.
For the linear regression case, the model generating the data is the same as above, but with $\epsilon_i\sim (1+\delta|x_{i,1}|)N(0,10)$, where $\delta$ is among $\{0, 0.5, 1\}$ to denote the correlation degree of model errors with the first variable $x_{i,1}$. For logistic regression, we assume the dataset is generated from the model: $Pr(y_i=1)=[1+\exp(-(x_i^{\top}\theta+\delta x_{i,6}^2))]^{-1}$, where $x_{i,6}$ is independently generated from $N(0,1)$ and $\delta$ is among $\{0, 0.5, 1\}$ to denote the degree of model mis-specification. 

We report the ratios of MSE of $\tilde{\theta}_n$ to $\hat{\theta}_n$ in Table \ref{MSE-comparison}. 
For LEV in linear regression, the inverse probability weighting is doing worse than the equal weighting when the models are correctly specified; however it outperforms the equal weighting when there are model mis-specifications, and the outperformance increases as the misspecification degree $\delta$ increases. For GRAD and Hessian, the inverse probability weighting is uniformly much better than the equal weighting, since GRAD and Hessian are  adaptive sampling methods which rely on the response. These observations follow that, (a) when sampling methods do not reply the response, the inverse probability weighting is robust if one is not quite sure whether the model is completely correct; (b) when sampling methods reply the response, the inverse probability weighting is a good choice no matter whether the model is correct.
\begin{table}[h]
\caption{MSE ratios between $\tilde{\theta}_n$ from equal weighting and $\hat{\theta}_n$ from the inverse probability weighting for linear/logistic regressions.
}
\begin{center}
\begin{tabular}{ c | c | c c c c c | c c c c c}
\hline
\multirow{2}{*}{sampling} & \multirow{2}{*}{$\delta$} &  0.005 & 0.01 & 0.02 & 0.04 & 0.08 &  0.005 & 0.01 & 0.02 & 0.04 & 0.08\\
\cline{3-12}
& & \multicolumn{5}{c|}{Linear regression} &  \multicolumn{5}{c}{Logistic regression}\\
\hline
 & 0 &  0.445 & 0.457 & 0.472 & 0.489 & 0.489 & - & - & - &- & -\\
LEV  & $0.5$ & 1.518 & 1.529 & 1.524 & 1.634  & 1.493 & - & - & - &- & - \\
& $1$ &  1.614  & 1.621  & 1.614  & 1.723  & 1.572 & - & - & - &- & - \\\hline
 & $0$ & 2.584 & 3.294 & 4.806 & 7.931 & 2.949 & 84.87 & 157.9 & 323.8 & 438.6 & 251.2 \\
GRAD & $0.5$ &  9.600 & 13.02 & 16.51 & 26.00 & 12.05 & 142.8 &  325.0 & 509.1 & 443.5 & 132.4 \\
& $1$ & 10.49 & 14.31 & 17.90 & 29.34 & 13.69 & 150.7 &  293.8 & 441.6 & 312.9 & 91.11 \\\hline
 & $0$ & 2.988 & 4.015 & 5.527 & 9.391 & 3.357 & 109.9 & 222.5 & 394.9 & 441.7 & 166.4 \\
Hessian & $0.5$ &  9.224 & 11.62 & 15.05 & 22.51 & 10.28 & 212.6 &  472.5 & 583.9 & 309.4 &  69.74 \\
& $1$ & 9.36 & 12.38 & 16.80 & 24.09 & 11.64 & 249.9 &  466.7 & 531.5 & 230.1 & 48.84 \\
\hline
\end{tabular}
\end{center}
\label{MSE-comparison}
\end{table}


\end{document}